%% file: arxiv-contract_predictions.tex
\documentclass[11pt,A4paper,fullpage]{article}

\usepackage{amsfonts}
\usepackage{fullpage}
\usepackage{graphicx,color}
\usepackage{xspace}
\usepackage{enumerate}
\usepackage{amssymb,amsmath}
\usepackage{amsthm}
\usepackage{subfig}

\usepackage[T1]{fontenc}
\usepackage[utf8]{inputenc}
\usepackage{authblk}
\usepackage{hyperref}

\newtheorem{theorem}{Theorem}

\newtheorem{corollary}[theorem]{Corollary}

\newtheorem{lemma}[theorem]{Lemma}

\newtheorem{property}[theorem]{Property}

\newcommand{\robust}{\ensuremath{{{\sc Robust}_p}}\xspace}

\newcommand{\acc}{\ensuremath{{\tt{acc}}}\xspace}

\begin{document}
\title{Contract Scheduling With Predictions\thanks{Research supported by the CNRS-PEPS project ADVICE and the Fondation Mathmatique Jacques Hadamard Programme Gaspard Monge PGMO. }}

\author[1]{Spyros Angelopoulos}
\author[2]{Shahin Kamali}

\date{}

\affil[1]{CNRS and Sorbonne   Universit\'e,  Laboratoire   d’Informatique  de  Paris   6, 4 Place Jussieu, Paris, France 75252.
{email: {\tt spyros.angelopoulos@lip6.fr}}
}

\affil[2]{Department of Computer Science, University of Manitoba, Winnipeg, MB, Canada. 
{email: {\tt shahin.kamali@umanitoba.ca.}}
}

\maketitle

\begin{abstract}
Contract scheduling is a general technique that allows to design a system with interruptible capabilities, 
given an algorithm that is not necessarily interruptible. Previous work on this topic has largely assumed
that the interruption is a worst-case deadline that is unknown to the scheduler. In this work, we study the setting in which there is a potentially erroneous {\em prediction} concerning the interruption.
Specifically, we consider the setting in which the prediction describes the time that the interruption occurs, as well as the setting in which the prediction is obtained as a response to a single or multiple binary queries. 
For both settings, we investigate tradeoffs between the robustness (i.e., the worst-case performance assuming adversarial prediction) and the consistency (i.e, the performance assuming that the prediction is error-free), both from the side of positive and negative results. 
\end{abstract}

\section{Introduction}
\label{sec:introduction}

One of the central objectives in the design of intelligent systems is the provision of {\em anytime} capabilities. In particular, several applications such as medical diagnostic systems and motion planning algorithms require that the system outputs a reasonably efficient solution given the unavoidable 
constraints on computation time. 
{\em Anytime algorithms} offer such a tradeoff between computation time and quality of the output. Namely, in an
anytime algorithm the quality of output improves gradually as the computation time increases. 
This class of algorithms was introduced first in~\cite{deliberation:boddy.dean} in the context of time-depending planning, as well as in~\cite{Horvitz:reasoning} in the context of flexible computation. 

\cite{RZ.1991.composing, DBLP:journals/ai/ZilbersteinR96} introduced a useful distinction between two different types of anytime algorithms. On the one hand, there is the class of {\em contract} algorithms, which describes algorithms that are given the amount of allowable computation time (i.e, the intended query time) as part of the input. However, if the algorithm is interrupted at any point before this ``contract time'' expires, the algorithm may output a result that is meaningless. 
On the other hand, the class of {\em interruptible} algorithms consists of algorithms whose allowable running time is not known in advance, and thus can be interrupted (queried) at any given point throughout their execution. 

Although less flexible than interruptible algorithms, contract algorithms typically use simpler
data structures, and are thus often easier to implement and maintain~\cite{steins}. Hence a natural question arises: how can one
convert a contract algorithm to an interruptible equivalent, and at which cost? This question can be addressed in an ad-hoc manner, depending on the algorithm at hand; however, there is a simple technique that applies to any possible contract algorithm, and consists of repeated executions of the contract algorithm with increasing runtimes (also called {\em lengths}). For example, consider a {\em schedule} of executions of the contract algorithm in which the $i$-th execution has length $2^i$. Assuming that an interruption occurs 
at time $t$, then the above schedule guarantees the completion of a contract algorithm of length at least $t/4$,
for any $t$. The factor 4 measures the performance of the schedule, and quantifies the penalty due to the repeated executions.

More formally, given a contract algorithm $A$, a {\em schedule} $X$ is defined by an increasing sequence $(x_i)$ in which $x_i$ is the length of the $i$-th execution of $A$. For simplicity, we call the $i$-th execution of $A$ in $X$ the $i$-th {\em contract}, and we call $x_i$ its {\em length}. The {\em acceleration ratio} of $X$, denoted by 
$\acc(X)$, relates an interruption $T$ to the length of the largest contract
that has completed by time $T$ in $X$, which we denote by $\ell(X,T)$, and is defined as
\begin{equation}
\acc(X)=\sup_T \frac{T}{\ell(X,T)}
\label{eq:acceleration}
\end{equation}

Intuitively, the acceleration ratio describes a trade-off between processor speed and resilience to interruptions. Namely, by executing the schedule $X$ to a processor of speed equal to $\acc(X)$, one obtains a system that is as efficient as a single execution of a contract algorithm that knows when the interruption will occur, but runs in a unit-speed processor. 

Contract scheduling has been studied in a variety of settings related to AI. 
It has long been known that the schedule $X=(2^i)$ has optimal acceleration ratio equal to 4~\cite{RZ.1991.composing}. Optimal schedules in multi-processor systems were obtained in~\cite{BPZF.2002.scheduling}.
The generalization in which there are more than one problem instances associated with the contract algorithm 
was first studied in~\cite{ZilbersteinCC03}, in which optimal schedules were obtained for a single processor. 
The more general setting of multiple instances and multiple processors was first studied in~\cite{steins}
and later in~\cite{aaai06:contracts}. \cite{soft-contracts} considered the problem in which the interruption is not a fixed deadline, but there is a ``grace period'' within which the system is allowed to complete the execution of the contract. Measures alternative to the acceleration ratio were proposed and studied in~\cite{ALO:multiproblem}. More recently,~\cite{DBLP:conf/ijcai/0001J19} studied contract scheduling in the setting in which the schedule is deemed ``complete'' once a contract reaches some prespecified end guarantees. 

Contract scheduling is an abstraction of resource allocation under uncertainty, in a worst-case setting. As such it has connections to other problems of a similar nature, such as online searching under the competitive ratio~\cite{steins,spyros:ijcai15,kupavskii2018lower}.

\subsection{Our setting: contract scheduling with predictions}
\label{subsec:setting}

Previous work on contract scheduling has mostly assumed that the interruption is unknown to the scheduler, and 
thus can be chosen adversarially, in particular right before a contract terminates. In practice, however, the scheduler may have a certain {\em prediction} concerning the interruption. Consider the example of a medical diagnostic system. Here, the expert may know that the system will be likely queried around a {\em specific time}, (i.e., prior to a scheduled surgery). Another possible prediction may describe a partition of time in {\em intervals} in which the system will likely be queried. In the example of the medical diagnostic system, it is more likely that the consultation will be required over a weekday, than over a weekend.

We study two settings that capture the above scenarios. In the first setting, there is a prediction $\tau$ concerning the interruption $T$. In the second setting, the prediction is in the form of answers to $n$
{\em binary queries}, where $n$ is a specified parameter. For example, a binary query can be of the form ``Will the interruption occur within a certain subset of the timeline?''. For both settings, the prediction is not necessarily trustworthy, and comes with an unknown {\em error} $\eta$.

The performance of the schedule is determined by two parameters: the first is the {\em robustness}, which is the worst-case acceleration ratio of the schedule assuming adversarial error (i.e., an adversary manipulating the prediction). The second is the {\em consistency} of the schedule, which is the acceleration ratio assuming that the prediction is error-free. In between these extremes, the acceleration ratio will be, in general, a function of the prediction error.
This follows the recent framework in machine learning of robust {\em online computation} with predictions, as introduced in~\cite{DBLP:conf/icml/LykourisV18} for the caching problem, and later applied in~\cite{NIPS2018_8174} for other online problems such as ski rental and non-clairvoyant scheduling.

Our paper differs from the above works in two important aspects. The first is related to the nature of the results.  More precisely, our aim is to complement the positive results obtained by specific schedules, with 
negative, i.e., impossibility results.  This is in the spirit of recent work~\cite{DBLP:conf/soda/Rohatgi20} which showed lower bounds on the competitive ratio of any caching algorithm, as a function of the prediction error, the cache size and the optimal cost. We are also interested in finding schedules that are {\em Pareto-efficient} with respect to the tradeoff between robustness and consistency. 
This is inspired by~\cite{DBLP:conf/innovations/0001DJKR20} which studied Pareto-efficient online 
algorithms with untrusted advice~\cite{DBLP:conf/innovations/0001DJKR20}.

The second difference is related to the nature of the problem we study. Unlike  ``natural'' online optimization problems in~\cite{DBLP:conf/icml/LykourisV18} and~\cite{NIPS2018_8174}, contract scheduling under the acceleration ratio poses certain novel challenges. Most notably, it is not the case that the performance improves monotonically as the error decreases. To see this, consider an interruption $T$, a prediction $\tau$ for $T$, and a schedule $X$ for prediction $\tau$. Suppose that a contract finishes right before $\tau$ in $X$: this is intuitively bad, because even with very small error, it is possible that $X$ barely misses to complete its largest contract by time $T$.  But it is also possible that if the error is very large, $T$ happens to occur right after another contract terminates in the schedule. This is a ``best-case'' scenario for the schedule: it completes a contract right on time. This observation exemplifies the type of difficulties we face. Another difficulty is that there may exist schedules that are Pareto optimal for the pair 
of consistency and robustness, but whose performance falls back to the worst-case acceleration ratio for {\em any} non-zero error. Such schedules are clearly undesirable, which is another challenge we must overcome.

\subsection{Results}
\label{subsec:results}

We first consider the setting in which the prediction $\tau$ is the interruption time $T$.
The prediction $\tau$ comes with an {\em error} $\eta \in [0,1]$ such that 
$T \in [\tau(1-\eta), \tau(1+\eta)]$. 
We show how to obtain a Pareto-optimal schedule by showing a reduction from an online
problem known as {\em online bidding}~\cite{ChrKen06}. This allows us to use, as black-box,
a Pareto-optimal algorithm of~\cite{DBLP:conf/innovations/0001DJKR20}, and obtain
a schedule with the same ideal performance. But there are two complications: this schedule cannot tolerate {\em any} errors (see the discussion above), and is also fairly complex. We give another simple schedule with the same robustness and consistency, and thus also Pareto-optimal. We then show how to extend this schedule to the realistic setting in which $\eta \neq 0$, and we complement the positive results with lower bounds on the performance of any schedule. 

In the second part we study the setting in which the prediction is in the form of answers to $n$ {\em binary queries}, for
some given parameter $n$, i.e., we would like to combine the advice of $n$ binary experts. Thus, the prediction is an $n$-bit string, and the prediction error $\eta\in [0,1]$ is defined as the fraction of the erroneous bits
in the string. First, we show an information-theoretic lower-bound on the best-possible consistency one can hope to achieve in this setting, assuming optimal robustness equal to 4. We then present and analyze a family of schedules, parameterized by the range of error that each schedule can tolerate. There are several challenges here: the analysis must incorporate several parameters such as the error $\eta$, the number of queries $n$ and the desired robustness $r$. Moreover, we need to define queries that are realistic and have a practical implementation. To this end, each query is a {\em partition} query of the form ``Does interruption $T$ belong to ${\cal T}$?'', where ${\cal T}$ is a subset of the timeline.

\paragraph{Other related work}
There are several recent works that study algorithms with ML predictions in a status of uncertainty. Examples include online rent-or-buy problems with multiple expert predictions~\cite{gollapudi2019online}, queuing systems with job service times predicted by an oracle~\cite{DBLP:conf/innovations/Mitzenmacher20}, online algorithms for metrical task systems~\cite{antoniadis2020online}, and online makespan scheduling~\cite{lattanzi2020online}. Clustering with noisy queries was studied in~\cite{NIPS2017_7161}.

Concerning contract scheduling, the work that is closest to ours is~\cite{ZilbersteinCC03},
in which there is stochastic information about the interruption, and the objective is to optimize the expected quality of the output upon interruption. The optimal scheduling policy in~\cite{ZilbersteinCC03} is based on a Markov decision process, hence no closed-form solution is obtained. More importantly, their schedule does not provide worst-case guarantees (i.e., a bound on the robustness), but only average-case guarantees for the given distribution, which is also assumed to be known.

\subsection{Preliminaries}
\label{subsec:preliminaries}

A contract schedule is defined by a sequence $X=(x_i)_{i\geq 1}$ of contract lengths, or {\em contracts}, where $x_i$ as  the {\em $i$-th contract} in $X$. We will always denote by $T$ the time at which an interruption occurs. 
We will make the standing assumption that an interruption can occur only after a unit time has elapsed.
With no prediction on $T$, the worst-case acceleration ratio of $X$ is given by~\eqref{eq:acceleration}; this is the {\em robustness} of $X$, which we denote by $r_X$, or simply $r$, if the schedule is implied.
With a prediction, the acceleration ratio of $X$ is simply defined as $T/\ell(X,T)$.  
Given a prediction, the {\em consistency} of $X$ is its acceleration ratio assuming $\eta=0$.  We will say that a schedule has {\em performance} $(r,s)$ if it has robustness $r$ and consistency $s$. In a {\em Pareto-optimal} schedule, these are in a Pareto-optimal relation.

Given a schedule $X=(x_i)$, it is easy to see that the worst-case interruptions occur infinitesimally
prior to the completion of a contract. Hence the following useful formula.
\begin{equation}
r_X= \sup_{i \geq 1} \frac{\sum_{j=1}^i x_j}{x_{i-1}},
\label{eq:acc.ratio.x}
\end{equation}
where $x_{0}$ is defined to be equal to -1. 

The class of {\em exponential} schedules describes schedules in which the $i$-th contract has length $a^i$, for some fixed $a$, which we call the {\em base} of the schedule. For several variants of the problem, there are efficient schedules in this class. The robustness of an exponential schedule 
with base $a$ is equal to $a^2/(a-1)$~\cite{ZilbersteinCC03}, and for $a=2$ the corresponding schedule has optimal robustness 4~\cite{RZ.1991.composing}. Let 
\[
c_r=\frac{r-\sqrt{r^2-4r}}{2} \ \textrm{and } b_r=\frac{r+\sqrt{r^2-4r}}{2},
\]
then it is easy to verify that for any given $r \geq 4$, an exponential schedule with base $a \in [c_r, b_r]$ has 
robustness at most 4. This fact will be useful in our analyses.

In the {\em online bidding} problem~\cite{ChrKen06}, we seek an increasing sequence $X=(x_i)$ of positive numbers (called {\em bids}) of minimum {\em competitive ratio}, defined formally as 
\begin{equation}
\sup_{u \geq 1}\frac{\sum_{j=1}^i x_j}{u} \ : x_{i-1} < u \leq x_i.    
\label{eq:cr.bidding}
\end{equation}
Without predictions, online bidding is equivalent to contract scheduling: given an increasing sequence $X=(x_i)$, both its acceleration ratio and its competitive ratio can be described by~\eqref{eq:acc.ratio.x}. We say that a bidding sequence has performance $(r,s)$ with a given prediction if it has robustness $r$ and consistency $s$ with respect to its competitive ratio.

\section{Interruption time as prediction}
\label{sec:interruption}

We first consider the setting in which the prediction $\tau$ describes the interruption time $T$. The prediction comes with an {\em error} $\eta \in [0,1]$, defined as follows. If $T\geq \tau$, then we define $\eta$ to be such that $T/\tau=(1+\eta)$, and  if $T\leq \tau$, then we define $\eta$ to be such that $T/\tau=(1-\eta)$. In the former case, we will say that the error is {\em positive}, otherwise we will say that the error is {\em negative}. Regardless of the sign of error, we have that $T \in [\tau(1-\eta), \tau(1+\eta)]$. 

We will also study settings in which the error $\eta$ is bounded by a quantity $H\leq 1$ which may or may not be known to the schedule. We thus distinguish between {\em $H$-oblivious} and {\em $H$-aware} schedules. 
Note that if $\eta$ is bounded by $H$ then 
\[
\tau(1-H) \leq \tau(1-\eta) \leq T \leq \tau(1+\eta) \leq \tau(1+H).
\]

We will first consider the ideal case in which either the prediction is error-free (hence the robustness is evaluated for $\eta=0$), or it is adversarially generated (hence the consistency is the worst-case acceleration ratio),

\begin{theorem} 
Suppose that for every $r\geq 4$, there is a sequence for online bidding that has performance $(r,s)$ for prediction equal to the target $u$. Then there is a contract schedule with the same guarantees for the setting in which the prediction is the interruption, and vice versa.
\label{thm:reduction}
\end{theorem}

\begin{proof}
We first give the reduction from contract scheduling to biding.
Namely, we have a prediction $\tau$ for the interruption, and let $X=(x_i)$ an $(r,s)$-competitive bidding sequence for prediction of a 
target $u=\tau$. Let $m$ be  the smallest index such that $x_m \geq \tau$ in $X$. We can assume, without loss of generality, that $x_m=\tau$, otherwise, by scaling down the bids by a factor $x_m/\tau$ we obtain a new sequence that is no worse than $X$, in both its robustness and its consistency. From the definition of consistency for $X$, we have that
\begin{equation}
\sum_{i=1}^m x_i \leq s \cdot x_m =s \cdot \tau.
\label{eq:red1}
\end{equation}
Moreover, from the definition of $r$-robustness we have that
\begin{equation}
\sum_{j=1}^i x_j \leq r \cdot x_{i-1}, \ \textrm{for all $i\geq 1$.}
\label{eq:red2}
\end{equation}
Consider now the schedule $X_\tau$ in which the lengths of the contracts are 
defined by the sequence $X_\tau=(x'_i)$, in which $x'_i=x_i/s$. Then in $X_\tau$, 
we have that the contract $x'_m$ is completed by time 
$(\sum_{i=1}^m x_i)/s\leq \tau$ (from~\eqref{eq:red1}), and the consistency of the schedule
is at most
$
\frac{\sum_{i=1}^m (x'_i)}{x'_m} \leq s,
$ 
again from~\eqref{eq:red1}. Last, note that since $X$ satisfies~\eqref{eq:red2}, so does $X_\tau$, and thus
the schedule must be $r$-robust. 

The reduction in the opposite direction follows along the same lines, by scaling up the contract lengths to as to obtain the bids.  
\end{proof}

From~\cite{DBLP:conf/innovations/0001DJKR20}, there is a Pareto-optimal bidding sequence, which satisfies the conditions of Theorem~\ref{thm:reduction} with $s=c_r$. This implies the following.
\begin{corollary}
For every $r\geq 4$, there is a contract schedule $X^*_\tau$ that has performance $(r,c_r)$, and this is Pareto-optimal.
\label{cor:pareto}
\end{corollary}

We also obtain the following corollary.
\begin{corollary}
For any $r$-robust schedule $X$, and any time $t$, it holds that $\ell(X,t)\leq t/c_r$. Moreover, for any $\epsilon>0$, there exists $i_0$ such that if $x_i=\ell$, with $i \geq i_0$, 
then the completion time of $x_i$ is at least $c_r \ell -\epsilon$.
\label{cor:sizes}
\end{corollary}

\begin{proof}
That $\ell(X,t) \leq t/c_r$ follows directly from Corollary~\ref{cor:pareto}, since otherwise we would get a schedule with performance that would dominate $(r,c_r)$, a contradiction. For the second part, suppose again, by way of contradiction that the statement were not true. Then from the reduction given in 
Theorem~\ref{thm:reduction}, there would exist an $r$-robust online bidding algorithm which, given arbitrarily 
large target $u$, would have cost at most $c_r u -\epsilon$. This would imply an $r$-robust algorithm for online bidding with consistency smaller than $c_r$, which contradicts the result of~\cite{DBLP:conf/innovations/0001DJKR20}.
\end{proof}

There are two issues here. The first is that the schedule obtained using the reduction to online bidding is fairly complex, because the bidding algorithm in~\cite{DBLP:conf/innovations/0001DJKR20} is quite complex. We can give instead, a different schedule, which is more intuitive and has the same performance, hence also Pareto-optimal. In particular, consider the geometric schedule $G=(b_r^i)$. Then there exists $\gamma<1$ such that
in the schedule $(\gamma b_r^i)$, there is a contract that completes at time precisely equal to $\tau$.

It is not difficult to see that this simple schedule, which we will denote by $X^*_\tau$, has also performance $(r,c_r)$, and thus is also Pareto-optimal, from Corollary~\ref{cor:pareto}. 
First, from the fact discussed in the introduction,
we know that the schedule $(b_r^i)$ is $r$-robust, and so is then $X_\tau^*$ due
to the scaling of all contracts by $\gamma$. Moreover, from the definition of $\gamma$, at time 
$\tau$, $X_\tau^*$ completes a contract of length $\gamma b_r^m$, for some $m$, and the consistency of the schedule is at most
\[
\sum_{i=1}^m \frac{\gamma b_r^i}{\gamma b_r^m} \leq \frac{b_r}{b_r-1}=c_r,
\]
where the last equality is a property that follows from the definitions of $b_r$ and $c_r$.

The second, and more significant issue, is that as in the case of the online bidding algorithm of~\cite{DBLP:conf/innovations/0001DJKR20}, in the presence of {\em any} error $\eta \neq 0$, the acceleration ratio of 
$X^*_\tau$ becomes as bad as its robustness $r$. This is because if $T=\tau-\epsilon$, a long contract in $X_\tau^*$ is not completed.

We will next adapt $X_\tau^*$ in order to obtain a more realistic schedule. The idea is to allow some ``buffer'' so that the schedule can tolerate mispredictions as a function of the buffer size. More precisely, for any $p\in (0,1)$, consider the schedule $X^*_{\tau(1-p)}$. 
The following lemma gives an upper bound on the performance of this parameterized, and $H$-oblivious schedule.

\begin{lemma}
For any $p\in (0,1)$, and $r\geq 4$, $X^*_{\tau(1-p)}$ is $r$-robust and has consistency 
$\min\{\frac{c_r}{1-p},r\}$. It also has acceleration ratio at most 
$\min \{\frac{c_r(1+\eta)}{(1-p)}, r\}$ for positive error, at most $\min \{\frac{c_r(1-\eta)}{(1-p)},r\}$ if
$\eta$ is negative error with $\eta \leq p$, and at most $r$, in every other case.  
\label{lemma:bidding.buffer}
\end{lemma}

\begin{proof}
First, note that by construction the schedule is guaranteed to be $r$-robust, so neither the acceleration ratio,
nor the robustness can exceed $r$.
With prediction $\tau$, $X^*_{\tau(1-p)}$ completes a contract of length 
$l=\frac{\tau(1-p)}{c_r}$ by time $T$ (this follows from the statement of the schedule, and 
Corollary~\ref{cor:pareto}). By definition,
the acceleration ratio of the schedule is at most $\frac{T}{l}$; moreover, for positive error 
$\eta$ we have that $T=\tau(1+\eta)$, whereas for negative error  $\eta \leq p$ we have that $T=\tau(1-\eta)$.
For no error, we have that $T=\tau$ (for which the consistency if evaluated).
Combining the above observations yields the lemma.
\end{proof}

The above result provides a tradeoff between the acceleration ratio of $X^*_{\tau(1-p)}$, and the range in which it is sufficiently good, as a function of the error. To illustrate this, consider the case of negative error: If $p$ is relatively small, then the schedule has good acceleration ratio for relatively small $\eta$ ($\eta <p$), which however can (and will) become as large as $r$, for a relatively big range of error, i.e, for $\eta >p$.

We now argue that  these tradeoffs are unavoidable, in any  $r$-robust and $H$-oblivious schedule $X$ with prediction $\tau$. Recall that $\ell(X,\tau)$ denotes the largest contract completed in the schedule by time $\tau$ in $X$, and let 
$p \in [0,1]$ be such that $\tau(1-p)$ is the completion time of this contract. From Corollary~\ref{cor:sizes} we know that $\ell(X,\tau) \leq \tau(1-p)/c_r$, hence for negative error $\eta \leq p$, the acceleration ratio is at least $\frac{c_r(1-\eta)}{1-p}$, and hence the consistency is at least $\frac{c_r}{1-p}$.
Moreover, there exists $x>0$ such that at time $\tau(1+x)$, the largest completed contract does not
exceed $\ell(X,\tau)$. Hence for positive error $\eta<x$, the acceleration ratio is at least $\frac{c_r(1+\eta)}{1-p}$. Last, since the schedule is $H$-oblivious, $T$ can occur at points right before a contract terminates, for all contracts that completed before $\tau$. In this latter case, the acceleration ratio will inevitably be as large as $r$, as $T$ becomes large.

For these reasons we will next consider {\em $H$-aware} schedules in which $\eta \leq H$ and $H$ is known. A natural schedule then is $X^*_{\tau(1-H)}$, in which the buffer $p$ is determined by $H$. Its performance is described in the following lemma, whose proof follows similarly to Lemma~\ref{lemma:bidding.buffer}, by setting $p=H$. 
\begin{lemma}
$X^*_{\tau(1-H)}$ is $r$-robust,  and has acceleration ratio at most 
$\min \{\frac{c_r(1+\eta)}{(1-H)},r\}$ for positive error, and at most $\min \{\frac{c_r(1-\eta)}{(1-H)},r\}$ for negative error.  
\label{lemma:bidding.buffer.bounded}
\end{lemma}

Since $\eta \leq H$, we have $\acc(X^*_{\tau(1-H)}) \leq \min\{\frac{c_r(1+H)}{1-H} ,r \}$. The next lemma shows that $H$ can take values in a certain range, as function of $c_r$,
for which no other $r$-robust schedule can be better. 

\begin{lemma}
For any $H$ that satisfies the condition $\frac{1+H}{1-H} <\sqrt{\frac{c_r+1}{{c_r}}} -\delta$, 
for any fixed $\delta>0$, 
the acceleration ratio of any $H$-aware $r$-robust schedule is at least 
$\min\{\frac{c_r(1+H)}{1-H} ,r \}$.
\label{lem:h.lower}
\end{lemma}
\begin{proof}
By way of contradiction, let $X$ denote an $H$-aware 
schedule that has acceleration ratio 
at most $\min\{\frac{c_r(1+H)}{1-H} ,r \}$. Then given prediction $\tau$, $X$  must
complete by time $\tau(1-H)$ a contract, say $x$, of length at least 
\[
\frac{\tau(1-H)^2}{c_r(1+H)}.
\]
From Corollary~\ref{cor:sizes}, the completion time of $x$ must be at least 
\[
c_r \cdot \frac{\tau(1-H)^2}{c_r(1+H)}-\epsilon = \frac{\tau(1-H)^2}{1+H}-\epsilon ,
\]
for arbitrarily small $\epsilon>0$, since $T$ can be arbitrarily large.
We now claim that $x$ is also the largest contract completed by time $\tau(1+H)$ in $X$. By way of 
contradiction, suppose that there is a contract $y$ that follows $x$, and which completes by time $\tau(1+H)$.
Note that $y$ must be at least as big as $x$. Then it must be that 
\[
\frac{\tau(1-H)^2}{1+H}-\epsilon + \frac{\tau(1-H)^2}{c_r(1+H)} \leq \tau(1+H), 
\] 
and since $\epsilon$ can be arbitrarily small and smaller than $\delta$, we arrive at a contradiction, concerning the assumption on $H$.   
Thus, if $T=\tau(1+H)$ (i.e., for positive $\eta=H$, the largest contract completed is $x$, and thus the acceleration ratio is at least $c_r (\frac{1+H}{1-H})^2 \geq c_r\frac{1+H}{1-H}$.
\end{proof}

We can also show that there is an even larger range for $H$ than that of Lemma~\ref{lem:h.lower} for which no
other schedule
can {\em dominate} $X^*_{\tau(1-H)}$, in the sense that no schedule can have as good an acceleration ratio as 
$X^*_{\tau(1-H)}$ on all possible values of $\eta \leq H$, and strictly better for at least one such value.  
\begin{lemma}
For any $H$ such that $\frac{1+H}{1-H} <\frac{c_r+1}{c_r}-\delta$, no $r$-robust $H$-aware schedule 
dominates $X^*_{\tau(1-H)}$.
\label{lem:dom}
\end{lemma}

\begin{proof}
By way of contradiction, suppose that there exists an $r$-robust schedule $X$ that dominates $X^*_{\tau(1-H)}$.
Then $X$ must complete a contract, say $x$, of length at least $\tau(1-H)/c_r$, by time $\tau(1-H)$. Using the same arguments as in the proof of Lemma~\ref{lem:h.lower}, it follows that $X$ does not complete any 
contract bigger than $x$ by time $\tau(1+H)$. This implies that for all possible values of error, $X$ has the 
same acceleration ratio as $X^*_{\tau(1-H)}$, which contradicts the dominance assumption.
\end{proof}

\medskip
\noindent
{\bf Example.}
To put the above results into perspective, let us consider the case $r=4$ (best robustness). 
Then $c_r=2$, and $X^*_\tau$ is 2-consistent, but can have acceleration ratio 4 for any 
$\eta\neq 0$. For given bound $H$, $X^*_{\tau(1-H)}$ has acceleration ratio at most $\min \{\frac{2(1+\eta)}{(1-H)},4\}$ for positive error, and at most $\min \{\frac{2(1-\eta)}{1-H},4\}$ if $\eta$ is negative error. Thus, an absolute upper bound on its  acceleration ratio is $\min\{\frac{2(1+H)}{1-H}, 4\}$, whereas its
consistency is $\min \{\frac{2}{(1-H)},4\}$. For any
$H<0.101$, no 4-robust $H$-aware schedule has better acceleration ratio.
Last, for $H < 0.2$, there is no 4-robust $H$-aware schedule that dominates $X^*_{\tau(1-H)}$.

\section{Binary predictions}
\label{sec:binary}

In this section, we study the setting in which the prediction is in the form of answers to $n$ binary queries
$Q_1, \ldots ,Q_n$ , for some given $n$. Hence, the prediction $P$ is an $n$-bit string, where the $i$-th bit is the answer to $Q_i$. It is worth pointing out that even a single binary query can be quite useful.
For example, it can be of the form ``Is $T \leq B$, for some given bound $B$''?, or ``Is $T\in [a,b]$, for some given $a,b$''? The prediction error $\eta \in [0,1]$ is the fraction of erroneous bits in $P$. We will assume, for simplicity, that the total number of erroneous bits, that is $\eta n$, is an integer. 

Our approach to this problem is as follows. Let ${\cal X}$ be a set of $r$-robust schedules. The prediction $P$ will help choose a good schedule from this class. For positive results, we need to define ${\cal X}$, and show how the prediction can help us choose an efficient schedule from $X$; moreover the prediction must have a practical interpretation, and must tolerate errors. For negative (i.e., impossibility) results, we need to show that any choice of $2^n$ $r$-robust schedules in ${\cal X}$ cannot guarantee consistency below a certain bound. Note that in this scheme, all schedules in ${\cal X}$ must be $r$-robust, because {\em any} schedule in 
${\cal X}$ can be chosen, if the prediction is adversarially generated.

We begin with a negative result, for the simple, but important case $r=4$, i.e., for optimal robustness. The following theorem gives a lower bound on the consistency.

\begin{theorem}
For any binary prediction $P$ of size $n$, any schedule with performance $(4,s)$ 
is such that $s \geq 2^{1+\frac{1}{2^n}}$.
\label{thm:lower}
\end{theorem}

\begin{proof}
We first give an outline of the proof, which is based on an information-theoretic argument. With $n$ binary queries, the prediction $P$ can only help us choose a schedule from a class ${\cal X}$ of at most $2^n$ $4$-robust schedules. 
Let $X_1, X_2, \ldots ,X_{2^n}$ describe these schedules. By way of contradiction, suppose we could guarantee consistency $S=2^{1+\frac{1}{2^n}} -\delta$, with $\delta>0$. We show that there exists an ordering
of these schedules with the following property, which we prove by induction: there is a set of $2^n-1$ interruptions, $T_2, \ldots ,T_{2^n-1}$ such that, for interruption $T_i$, with $i\in [2,2^n-1]$, no schedule of rank at most $i+1$ in the ordering can guarantee consistency $S$. This means that for interruption  $T_{2^n-1}$, no schedule in ${\cal X}$ can guarantee robustness $S$, a contradiction. 

We now proceed with the technical details of the proof. 
Let $S=2^{1+\frac{1}{2^n}}$. By way of contradiction, suppose there is a schedule $Z$, which is 4-robust, and
with $n$-bit, error-free prediction $P$ which has robustness $S-\delta$, for some $\delta>0$. We will rely to the following  information-theoretic argument: with $n$-bit prediction, $Z$ can only differentiate between a set 
${\cal X}$ of $2^n$ schedules. Each schedule in ${\cal X}$ must be 4-robust, otherwise $Z$ cannot be 4-robust either. Let $X_1, \ldots X_{2^n}$ denote the $2^n$ schedules in ${\cal X}$, and we denote by
$x_{i,l}$ the length of the $i$-th contract in $X_l$. We also define $T_{i,l}= \sum_{j=1}^i x_{i,l}$ as the completion time of the $i$-th contract in $X_l$. We will say that for a given interruption $T$, $Z$ {\em chooses} 
schedule $X_l$ in ${\cal X}$ if the prediction $P$ points to this schedule. Without loss of generality, we can
assume that $Z$ will always choose a schedule that has completed the {\em largest} contract completed by time 
$T$, among all schedules in ${\cal X}$.

Let us fix some index $i \in {\mathbb N}^+$. There exists a schedule in $\{X_2, \ldots ,X_{2^n}\}$ that has completed the {\em largest} contract by time $T_{i,1}$ among all these schedules. Without loss of generality, we can assume that this schedule is $X_2$ (by re-indexing the schedules), and we denote by $\bar{i}_2$ the index of this largest contract in $X_2$.
Inductively, for all $l \in [2,2^n-1]$, there has to be a schedule in $\{X_{l+1}, \ldots X_{2^n}\}$ which has completed the largest contract by time $T_{\bar{i}_l,l}$ among all these schedules. Again, without loss of generality, we can assume that this schedule is $X_{l+1}$, and we denote by $\bar{i}_{l+1}$ the index of this largest contract in $X_{l+1}$. We will say that ${\cal X}$ is {\em ordered for index $i$} if its schedules obey the above properties. 

We will use the following helpful property concerning 4-robust schedules. This is proven in~\cite{DBLP:conf/innovations/0001DJKR20} in the context of online bidding, but also applies to contract scheduling, from the discussion in the Preliminaries.
\begin{lemma}
For every 4-robust schedule $X=(x_i)$, and every $\epsilon>0$ there exists $i_0$ such that 
$\sum_{j=1}^i x_j \geq (2-\epsilon) x_i$, for all 
$i\geq i_0$ 
\label{lemma:helper}
\end{lemma}

For every $\epsilon>0$, we can then choose $i_0$ sufficiently large such that the conditions of 
Lemma~\ref{lemma:helper} apply to $X_1$, but also to all schedules in ${\cal X}$, since otherwise some schedules in ${\cal X}$ would not have finite robustness. Consider then any fixed $i\geq i_0$, and let ${\cal X}$ be ordered for $i$. 
To simplify a bit the notation, let $y_l$ be equal to $x_{\bar{i}_l,l}$.
Then from Lemma~\ref{lemma:helper} we have that 
\begin{equation}
T_{\bar{i}_l,l} \geq (2-\epsilon) y_l, \ \textrm{for all $l \in [1,2^n]$, and $i \geq i_0$.} 
\label{lemma:2.helper}
\end{equation}

Note that $\epsilon$ can be chosen to be arbitrarily small, for sufficiently large $i_0$. To simplify the proofs, in what follows we will assume that Lemma~\ref{lemma:helper} holds with $\epsilon=0$. We can make this assumption without affecting correctness, because $\delta>0$ is fixed, and $\epsilon$ can be chosen arbitrarily small than 
$\delta$.  

We will prove the following property, which will yield the proof of the theorem. This is because for $l=2^n$, the property guarantees that $Z$ cannot choose any schedule in ${\cal X}$ so as to guarantee consistency strictly less than $S$, a contradiction. We will again simplify a little the notation. We will define $z_l$
to be equal to $x_{\bar{i}_l-1,l}$, i.e., the contract immediately preceding $y_l$ in $X_l$. We will 
also define $\alpha$ to be equal to $x_{i-1,1}$. This contract plays a special role in the proof, as we will see.

\begin{property}
Suppose that $Z$ has consistency at most $S-\delta$, and that is ordered for $i$. Then there exists $i \geq i_0$ such that for every $l \in [2,2^n]$, it must be that 
$y_l \geq 2^{1-\frac{l-1}{2^n}} \alpha$. Moreover, for every $l \in [2,2^n]$, if an interruption occurs at a time infinitesimally earlier than
$T_{\bar{i}_l,l}$, then $Z$ cannot choose any schedule in $\{X_1, \ldots X_l\}$.
\label{prop:lb}
\end{property}

\begin{proof}[Proof of Property~\ref{prop:lb}]
By induction on $l$. We begin with the base case, namely $l=2$. 
 Let $s=S-\delta$ denote the robustness of $Z$.

First, we will prove the lower bound on $y_2$. Recall that $X_1$ is 4-robust; furthermore, we know that no schedule can be better than 4-robust. This implies that
\[
\sup_{i \geq i_0} \frac{T_{i,1}}{x_{i-1,1}} = \sup_{i \geq i_0} \frac{T_{i,1}}{\alpha} \geq 4.
\]
By way of contradiction, suppose that $y_2<2^{1-\frac{1}{2^n}} \alpha$. Consider an interruption infinitesimally earlier than $T_{i,1}$. There are two possibilities: either $Z$ chooses $X_1$ or it chooses
$X_2$ (since ${\cal X}$ is ordered for $i$, this is the best schedule among $X_2, \ldots X_{2^n}$). In the former
case, we have that 
\[
s \geq \frac{T_{i,1}}{\alpha},
\]
whereas in the latter case we have that 
\[
s \geq \frac{T_{i,1}}{y_2} \geq \frac{T_{i,1}}{2^{1-\frac{1}{2^n}} \alpha}.
\]
These two possible cases apply for all $i \geq i_0$. Hence we obtain that 
\[
s \geq \sup_{i \geq i_0} \frac{T_{i,1}}{2^{1-\frac{1}{2^n}} \alpha} \geq \frac{4}{2^{1-\frac{1}{2^n}}}=2^{1+\frac{1}{2^n}}=S,
\]
a contradiction. 

Next, we will show that there exists an interruption such that $Z$ cannot choose either $X_1$ or $X_2$. Recall that by definition, $T_{\bar{i}_2,2}$ is the completion time of 
$y_2$, hence from Lemma~\ref{lemma:helper} we have that 
\[
T_{\bar{i}_2,2} \geq 2 y_2 \geq 2 2^{1-\frac{1}{2^n}} \alpha = 2^{2-\frac{1}{2^n}} \alpha.
\]
Consider an interruption infinitesimally earlier than $T_{\bar{i}_2,2}$. Suppose, by way of contradiction, that 
$Z$ chooses $X_1$. Then it must be that
\[
\frac{T_{\bar{i}_2,2}}{x_{i-1,1}} \leq s <S \Rightarrow S > 2^{2-\frac{1}{2^n}}, 
\]
a contradiction.

Next, suppose, again by way of contradiction, that $Z$ chooses $X_2$. Note that by definition, the largest contract completed by the above interruption in $X_2$ is $z_2$, and since $z_2$ cannot exceed 
$y_2$, from Lemma~\ref{lemma:helper} we obtain that 
\[
z_2 \leq \frac{ T_{\bar{i}_2,2}-y_2}{2}.
\]

Therefore, using again Lemma~\ref{lemma:helper} we infer that
\begin{align}
\frac{T_{\bar{i}_2,2}}{z_2} < S &\Rightarrow S > 2 \frac{T_{\bar{i}_2,2}}{T_{\bar{i}_2,2}-y_2 } \nonumber \\
&\geq 
4 \frac{y_2}{T_{\bar{i}_2,2}-y_2 } \geq 8 \frac{y_2}{T_{\bar{i}_2,2}} 
\label{eq:base.1}
\end{align}

However, we know that since $X_1$ is 4-robust, then $T_{i,1} \leq 4 \alpha$. Since $Z$ is ordered for $i$,
we know that $T_{\bar{i}_2,2} \leq T_{i,1}$, thus  $T_{\bar{i}_2,2} \leq 4 \alpha$. Combining with the above
inequality we obtain that 
\begin{equation}
S > 2 \cdot y_2 \geq 2 \cdot 2^{1-\frac{1}{2^n}}= 2^{2-\frac{1}{2^n}}, 
\label{eq:base.2}
\end{equation}
a contradiction. This concludes the base case.

For the induction hypothesis, suppose that the property holds for $l-1$. We will next show that it holds for $l$.

First, we will prove the lower bound on $y_l$. Consider an interruption infinitesimally earlier than 
$T_{\bar{i}_{l-1},l-1}$. From the induction hypothesis, we know that for this interruption, $Z$ must choose a schedule in $\{X_l, \ldots X_{l-1}\}$. The largest contract finished by that time is $y_l$. Thus it must be 
that
\begin{align*}
\frac{T_{\bar{i}_{l-1},l-1}}{y_l} &\leq S \Rightarrow y_l \geq \frac{T_{\bar{i}_{l-1},l-1}}{S} \nonumber \\
& \geq \frac{2y_{l-1}}{y_l} \tag{From Lemma~\ref{lemma:helper}} \\
&\geq 2\frac{2^{1-\frac{l-2}{2^n}} \alpha}{2^{1+\frac{1}{2^n}}}  \tag{From the induction hypothesis} \\
&=2^{1-\frac{l-1}{2^n}} \alpha.
\end{align*}

Consider now an interruption infinitesimally earlier than $T_{\bar{i}_{l},l}$. Suppose first, by way of contradiction, that $Z$ chooses $X_1$. Then it must be that
\begin{align*}
\frac{T_{\bar{i}_{l},l}}{x_{i-1},1} &< S \Rightarrow S \geq \frac{2y_l}{\alpha} \tag{From Lemma~\ref{lemma:helper}} \\
&\geq 2 \frac{2^{1-\frac{l-1}{2^n}} \alpha}{\alpha} \tag{Since $y_l \geq 2^{1-\frac{l-1}{2^n}} \alpha$} \\
&\geq 2 2^{1-\frac{2^n-1}{2^n}} \tag{Since $l \leq 2^n$} \\
= 2^{1+\frac{1}{2^n}},
\end{align*}
a contradiction. 

Suppose then, again by way of contradiction, that $Z$ chooses one of $X_2, \ldots X_{l-1}$, say $X_m$. We will arrive to a contradiction by applying an argument similar to the one we used for the base case. Namely, it must be that 
\[
\frac{T_{\bar{i}_{l},l}}{z_m} \leq S, \ \textrm{and } \ z_m \leq \frac{T_{\bar{i}_{l},l}-y_m}{2},
\]
from which we obtain that
\[
S \geq 2\frac{T_{\bar{i}_{l},l}}{T_{\bar{i}_{l},l}-y_m},
\]
and using the same argument as in~\eqref{eq:base.1} and~\eqref{eq:base.2} it follows that 
$S\geq 2^{2-\frac{1}{2^n}}$, a contradiction. 

We conclude that for the above defined interruption, $Z$ cannot choose a schedule in $\{X_1, \ldots X_l\}$.
This completes the inductive step, and the proof of the property.
\end{proof}
As explained earlier the above property suffices to prove the result.
\end{proof}

We complement Theorem~\ref{thm:lower} with the following positive result. Consider the set ${\cal X}=\{X_i, i \in [0,2^n-1]\}$ of schedules, in which $X_i=(x_{j,i})_{j\geq 1}$ is defined by $x_{j,i}=d^{j+\frac{i}{2^n}}$, for $d>1$ that we will choose later. In words, $X_i$ is a near-exponential schedule with base $d$,
and a scaling factor equal to $d^\frac{i}{2^n}$. The prediction $P$ then chooses an index, in $[0,2^n-1]$, of the schedule in this ${\cal X}$. We call {\sc Ideal} the schedule obtained from ${\cal X}$ with prediction $P$.

\begin{theorem} 
For every $r\geq 4$, define
$d=b_r$, if $r \leq \frac{(1+2^n)^2}{2^n}$, and $d=1+2^n$, otherwise.
Then {\sc Ideal} has performance $(r,d^{1+\frac{1}{2^n}}/(d-1))$.
\label{thm:k.upper}
\end{theorem} 

\begin{proof}
The worst-case interruptions occur infinitesimally earlier than the completion time of a contract in ${\cal X}$.
Specifically, let $T=\sum_{i=1}^j x_{{i,l}}-\epsilon$ be an interruption right before the $j$-th contract in 
$X_l$ completes. We have 
\[
\sum_{i=1}^j x_{{i,l}} =\sum_{i=1}^j d^{i+\frac{l}{2^n}} \leq d^\frac{l}{2^n} \frac{d^{j+1}}{d-1}.
\]
We consider two cases. If $l\neq 0$, then the largest contract completed by time $T$ in {\sc Ideal} is 
$x_{j,l-1}$ of length $d^{j+\frac{l-1}{2^n}}$. The prediction $P$ chooses this schedule, and the consistency is at most $T/x_{j,l-1} \leq \frac{d^{1+\frac{1}{2^n}}}{d-1}$. If $l=0$, then the largest such contract is contract $x_{j-1,2^n-1}$ of schedule $X_{2^n-1}$ which has length $d^{j-1+\frac{2^n-1}{2^n}}=d^{j-\frac{1}{2^n}}$. Again, the prediction chooses this schedule, and the consistency is at most $T/x_{j-1,2^n-1} \leq 
\frac{d^{1+\frac{1}{2^n}}}{d-1}$.

Moreover, we require that each schedule in ${\cal X}$ is $r$-robust, or equivalently, that 
$\frac{d^2}{d-1} \leq r$, as follows from the robustness of exponential schedules (see Preliminaries).

Thus we best value of $d$ is such that 
\[
\frac{d^2}{d-1} \leq r \ \textrm{ and } \frac{d^{1+\frac{1}{2^n}}}{d-1} \ \textrm{is minimized}.
\]
Using standard calculus, it follows that the optimal choice of $d$ is as in the statement of the theorem. 
\end{proof}

For example, for $r=4$, {\sc Ideal} has performance $(2^{1+\frac{1}{2^n}},4)$, which matches Theorem~\ref{thm:lower}, and is, therefore, Pareto-optimal.

{\sc Ideal}, as its name suggests, is not a practical schedule: a single error in one of the queries can make its acceleration ratio as bad as its robustness. Intuitively, this occurs because the $n$ queries implement a type of ``binary search'' in the space of all $2^n$
schedules in ${\cal X}$, and which is not robust to errors.
We will instead propose a family of schedules, which we call \robust, where $p \in [0,1]$ is a parameter that defines the range of error that the schedule can tolerate (this will become more clear shortly). More precisely, we will define a class of schedules 
${\cal X}$, and the prediction $P$ will be the index of one of these schedules. However, this time there are only $n$ schedules in ${\cal X}$ instead of $2^n$, as in the case of {\sc Ideal}. Each  $X_i \in {\cal X}$ is defined as $X_i=(x_{j,i})_{j \geq 0}=d^{j+\frac{i}{n}}$, with $i \in [0,n-1]$, and again $d>1$ to be determined later.

We now describe the $n$ queries that comprise the prediction $P$. Each query $Q_i$, for $i \in [0,n-1]$ is of
the form ``Is the best schedule, for the given interruption in $\{X_0, \ldots ,X_i\}$?''. Note that the queries obey a monotonicity property: if $Q_{i}$ is ``no'', and $Q_{i+1}$ is ``yes'', we know an error has occurred in one of these queries. It is also important to note that each of the queries $Q_i$ has an equivalent statement of the following form: ``Does the interruption $T$ belong to a subset $S_i$ of the timeline?''. Thus each query asks whether $T$ falls in a certain partition of the timeline, which has a more natural, and practical interpretation.

If there were no errors (i.e., for $\eta=0$), then the best schedule in ${\cal X}$ would be the number of 
``no'' responses to the $n$ queries, minus one to account for indexing from 0. However, in the presence of errors, one needs to be careful, because, once again, a single error can have an enormous impact. For this reason,
\robust uses the parameter $p$.  In particular, it chooses schedule $X_m$, where 
$m$ is defined as $(N-1-pn) \bmod n$  and $N$ is the number of ``no'' responses (again, for convenience we will assume that $pn$ is integral). In words, \robust chooses a schedule of index ``close and above'', in the cyclic order of indices, to an index that would correspond to an error-free prediction. The following theorem bounds the 
performance of \robust, and shows how to choose the base $d$. We make two assumptions: that $\eta \leq p$ (thus \robust can only tolerate up to $p$ fraction of query errors), and that $p \leq 1/2$ (otherwise, in the worst case, the queries are too ``corrupt'' to be of any use).

\begin{theorem}
For every $r\geq 4$, define $K$ to be equal to $\frac{2pn+1}{n}$, and $d$ to be equal to  
$b_r$, if $r\leq (1+K)^2/K$, and $1+K$, otherwise. Then  
\robust is $r$-robust and has acceleration ratio at most  $\frac{d^{1+\frac{1}{n}+2p}}{d-1}$, assuming 
$\eta \leq p \leq 1/2$.
\label{thm:oof}
\end{theorem}

\begin{proof}
For interruption $T$, let $l$ denote the index of the best schedule in ${\cal X}$. From the structure 
of ${\cal X}$, this means that, in worst-case,
$T$ occurs right before the completion of a contract, say $j$, in the schedule $X_{(l+1) \bmod n}$. We will consider the case $l \neq n-1$, thus $(l+1) \bmod n =l+1$; the outlier case $l=n-1$ follows similarly, but with a slightly different argument (namely, the worst case interruption occurs right before the completion time of contract $j+1$ of $X_0$). We 
express this interruption as 
\[
T=\sum_{i=1}^j x_{i,l+1} = \sum_{i=1}^j d^{i+\frac{l+1}{n}} \leq \frac{d^{j+1+\frac{l+1}{n}}}{d-1}.
\]
Let $m$ denote the index chosen by \robust, as defined earlier. The crucial observation is that in a cyclic
ordering of the indices, $m$ and $l$ are within a distance at most  $(\eta+p)n$. Here, a distance of at most $\eta n$ is due to the maximum number of erroneous queries, and an additional distance of at most $pn$ is further incurred by the algorithm.  Since $\eta \leq p$, 
they are within a distance at most $2pn$.  

We will give a lower bound on the largest contract length, say $L$ completed by time $T$ in \robust. We consider two cases. First, suppose that $m\leq l$, then by the structure of ${\cal X}$, $L$ is at least the length 
$x_{j,l-2pn}=d^{j+\frac{l-2pn}{n}}$. Next, suppose that $m>l$. In this case, $L$ is at least the length 
of $x_{j-1, n+l-2pn}=d^{j-1+\frac{n+l-2p n}{n}}=d^{j+\frac{l-2p}{n}}$. In both cases we conclude that $L\geq d^{j+\frac{l-2p n}{n}}$. Therefore the acceleration ratio is at most
$T/L \leq \frac{d^{1+\frac{1}{n}+2p}}{d-1}$. We now want to find $d$ such that $d^2/(d-1)\leq r$ and 
$\frac{d^{1+\frac{1}{n}+2p}}{d-1}$ is minimized. Using standard calculus, it follows that the best choice of $d$ 
is as in the statement of the theorem.
\end{proof}

For example if $r=4$, then for any given $p \leq 1/2$, \robust is 4-robust, can tolerate at most a 
$p\leq 1/2$ fraction of erroneous responses, and has acceleration ratio at most $2^{1+\frac{1}{n}+2p}$.
We can interpret the result of the theorem in two ways. First, one can use $p$ as a hedging parameter: with larger $p$, better tolerance to errors can be achieved, at the expense however of the acceleration ratio (akin to 
Lemma~\ref{lemma:bidding.buffer} and the discussion following it). Second, the acceleration ratio 
improves rapidly as a function of the numbers (not as rapidly as in {\sc Ideal}, but still very fast).

\section{Experimental results}
\label{sec:experiments}
\input{experiments-main-arxiv}

\section{Conclusion}

It is intriguing that a problem with a very simple statement, namely contract scheduling under the acceleration ratio, turns out to be quite challenging in the setting of predictions. We explored the tradeoffs between the prediction accuracy, the acceleration ratio, the consistency and the robustness of schedules in two natural settings of prediction. In future work, we would like to study the multi-instance setting, as discussed in the introduction, for which a lot of work has been done in the standard framework of no predictions. 

Another direction is to investigate connections between contract scheduling and {\em online searching on the line}, with untrusted hints. A very recent work~\cite{angelopoulos2020online} studied this problem assuming that the hints are either adversarially generated, or trusted and thus guaranteed to be correct. The techniques we developed and the results we showed in this paper should be readily applicable in searching with noisy, erroneous hints, given the known connections between contract scheduling and searching under the competitive ratio~\cite{steins,spyros:ijcai15}.

\bibliographystyle{plain}
\bibliography{targets-arxiv}

\end{document}

%% file: experiments-main-arxiv.tex
\input{tables.tex}
In this section, we present the experimental evaluation of our schedules. We use exponential schedules (without any prediction) as the baseline for our comparisons. Recall that for any $r\geq 4$, any exponential schedule $(a^i)$ with base $a \in [c_r,b_r]$ has robustness at most $r$. For the special but important case of $r=4$, there is only one such schedule with $a=2$. We report results for $r=4$, 
but we note that for $r>4$ the experiments show the same trends.

\subsection{Interruption time as prediction}

We model $\tau\in [T-H,T+H]$ to be a random, normal variable with mean $T$ and standard deviation 1, such that
$\eta \leq H$. Recall that an $H$-aware schedule knows $H$, whereas an $H$-oblivious one does not.
Figure~\ref{NewAccRatioNormalNoiseEta01} depicts the average acceleration ratio (y-axis) of the 
schedule $X^*_{\tau(1-p)}$ for different values of $p$, as a function of the interruption time $T$ (x-axis), 
for fixed $H=0.1$. The plot depicts the performance of four schedules: the $H$-aware schedule, in which $p=H$, and three $H$-oblivious schedules for $p=0.05$, $p=0.2$ and $p=0.3$. We run the experiment over 1,000 evenly spaced values of the interruption time in the interval $[2,2^{20}]$. For each value of the interruption time, the expectation is taken over 1,000 random values of the error. 

The figure shows that the $H$-aware schedule ($p=0.1$) has an advantage over the schedules with different values of $p$. In particular, the expected value of the acceleration ratio of this schedule is around 2.23 for all values of the interruption $T$, compared to acceleration ratios of 2.41 for the schedule with buffer smaller than $H$ ($p=0.05)$ and ratios 2.49 and 2.85 of the schedule whose buffer is larger than $H$ ($p=0.2,0.3$, respectively). As $p$ decreases, the fluctuation of the acceleration ratio due to the random noise increases, since the interruption becomes closer to the completion time of a contract. 

\begin{figure}[!b]
\centering
\hspace*{-8mm}\includegraphics[width = 0.8\columnwidth]{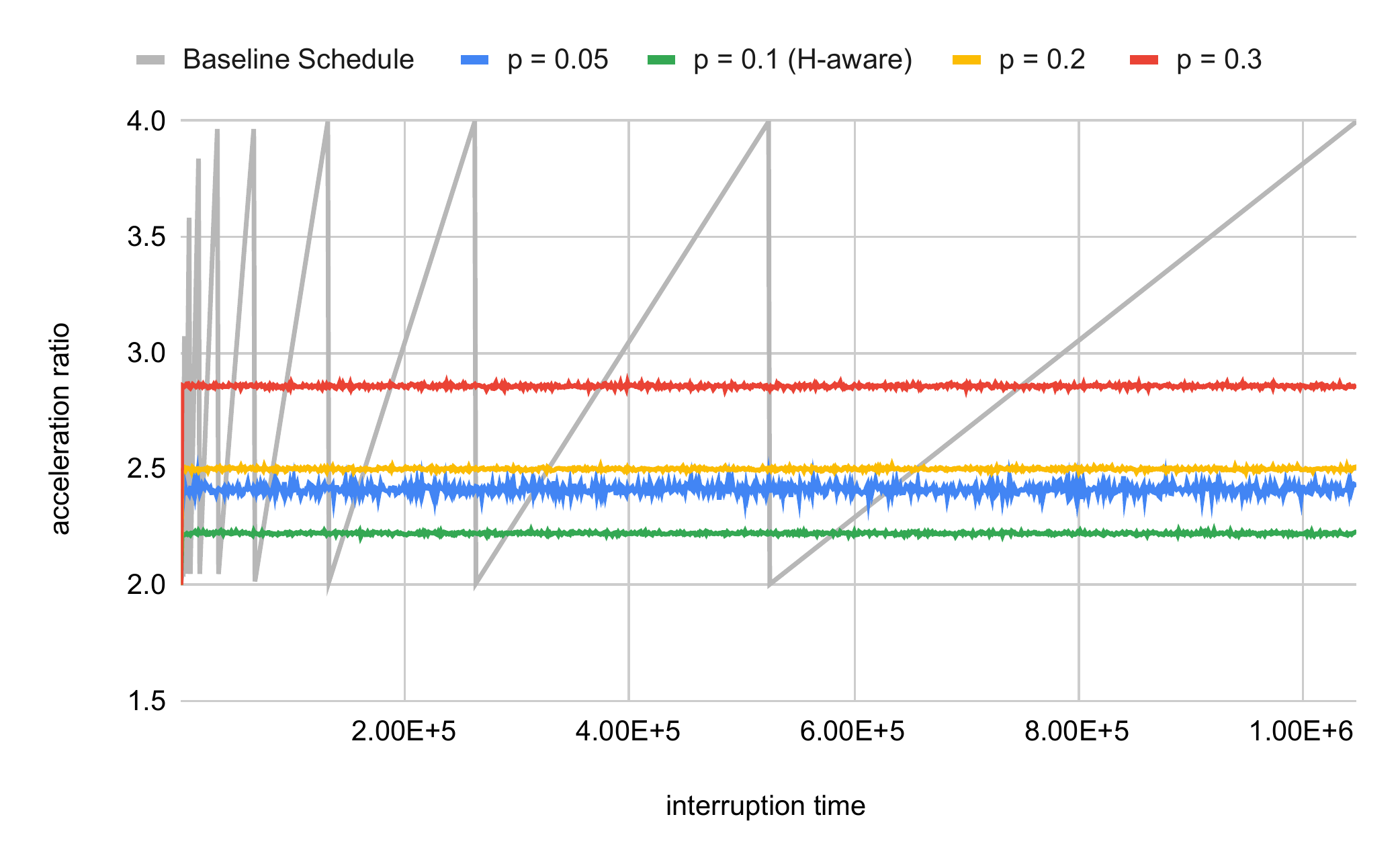}
\caption{Acceleration ratios of $X^*_{\tau(1-p)}$, for $H=0.1$.} 
\label{NewAccRatioNormalNoiseEta01}
\end{figure}

As Figure~\ref{NewAccRatioNormalNoiseEta01} shows, our schedules with predictions do not outperform the baseline algorithm for every interruption. This is to be expected, since there is no schedule that can dominate any other schedule. More precisely, even a schedule of very bad robustness (e.g., a schedule with a huge contract early on)
will have excellent acceleration ratio for some range of interruptions (e.g., for certain interruptions before the 
completion time of the huge contract). Nevertheless, we can quantify the advantage of the schedules with predictions, as shown in Table~\ref{tableInterruptionAsAdvice}. The table depicts the percentage of interruptions in $[2,2^{20}]$ for which $X^*_{\tau(1-p)}$ outperforms the baseline schedule, as well as the percentage of interruptions for which the improvement is significant (at least by 20\%).
As expected, the $H$-aware schedule yields the best improvements, but even the $H$-oblivious schedules tend to perform much better than the baseline schedule. The conclusion is that while 
$H$-awareness yields a clear improvement, it is not indispensable.

Similar conclusions can be drawn for different values of $H$, but as $H$ increases, the acceleration ratios of the schedules $X^*_{\tau(1-p)}$ also smoothly increase, as expected. In addition, similar results are obtained for $\tau$ uniformly at random in $[T-H,T+H]$.

\begin{table}[!h]
        \centering
\scalebox{.85}{        
\begin{tabular}{|p{0.3\columnwidth}|p{0.16\columnwidth}|p{0.16\columnwidth}|p{0.15\columnwidth}|p{0.15\columnwidth}|}
\hline 
  & $p = 0.05$ & $p = 0.1$ & $p = 0.2$ & $p = 0.3$ \\%
\hline \hline
 improvement & 79.22\% & 88.71\% & 74.73\% & 57.04\%  \\
\hline 
 strong improv. & 55.24\% & 66.43\% & 50.05\% & 28.47\% \\
 \hline
\end{tabular}
}
\caption{Percentage of interruptions in $[2,2^{20}]$ for which $X_{\tau(1-p)}$ outperforms the baseline schedule. 
}\label{tableInterruptionAsAdvice}        
\end{table}

\subsection*{Binary predictions}
We evaluate experimentally the performance of \robust (as mentioned earlier, {\sc Ideal} is not a practical schedule, and thus we do not implement it). We fix the number $n$ of queries to be equal to 100, and as in the previous setting, we also set $H=0.1$. Given a binary prediction of size 100, we generate a noisy prediction by flipping a fraction $\eta$ of the 100 bits (rounded down) where $\eta$ is chosen uniformly at random in $[0,H]$. 
Figure~\ref{new_query_acc_ratios} depicts the average acceleration ratio (y-axis) of \robust for different values of the parameter $p$, as a function of the interruption time $T$ (x-axis). As earlier, the expectation is taken over 1,000 random values of the error, and the interruption time takes values in the interval $[2, 2^{20}]$.

We consider \robust with four values of the parameter $p$, namely $p\in \{0.05, 0.1, 0,2, 0.3\}$. Note that the theoretical upper bound of Theorem~\ref{thm:oof} applies only if $p \geq 0.1$ in this setting. For such values of $p$, the acceleration ratio is a ``saw-like'' function of the interruption. There are some ``critical'' interruptions at which the acceleration ratio drops, then gently increases until the next critical interruption, as shown in 
Figure~\ref{new_query_acc_ratios}. The acceleration ratio of \robust also increases with $p$, as predicted by Theorem~\ref{thm:oof}, but is much smaller than the baseline acceleration ratio; for instance, for $p=0.3$, it fluctuates in the interval 
$[2.4,2.6]$. Note also that even for $p=0.05<H$, \robust performs better than the baseline schedule, which is interesting because such a case is not captured by Theorem~\ref{thm:oof}. This implies that \robust may work in practice for a wider range of values of $p$ than predicted by the theorem, and that \robust need not be $H$-aware to perform well. 

\begin{figure}[!t]
\centering
\hspace*{-4mm}\includegraphics[width = 0.8\columnwidth]{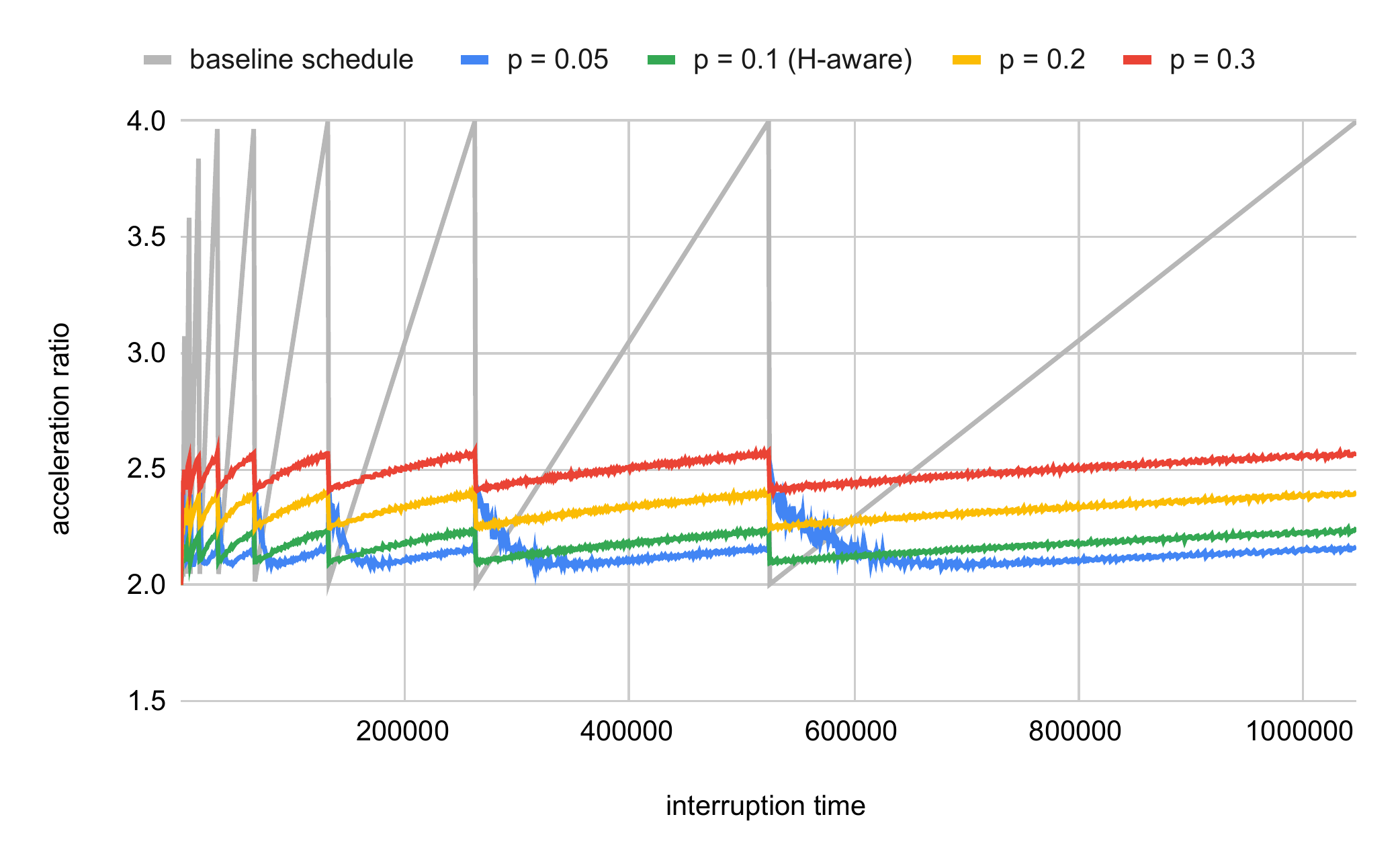}
\caption{Acceleration ratios of \robust, for $H=0.1$.} 
\label{new_query_acc_ratios}
\end{figure}

To quantify the above observation, in Table~\ref{tableRobust} we report the performance gain of \robust for different values of $p$, and fixed $H=0.1$. Once again, the table shows the percentage of interruptions in the range $[2,2^{20}]$ for which \robust outperforms the baseline schedule, as well as the percentage in which the performance gain is significant (at least 20\%).

\begin{table}[!h]
        \centering
\scalebox{.85}{        
\begin{tabular}{|p{0.3\columnwidth}|p{0.17\columnwidth}|p{0.17\columnwidth}|p{0.15\columnwidth}|p{0.15\columnwidth}|}
\hline  
  & $p = 0.05$ & $p = 0.1$ & $p = 0.2$ & $p = 0.3$ \\%
\hline \hline
 improvement        & 89.81\% & 94.25\% & 86.07\% & 77.07\%  \\
\hline 
 strong improv. & 74.33\% & 70.98\% & 60.94\% & 49.95\% \\
 \hline
\end{tabular}
}
\caption{Percentage of interruptions in $[2,2^{20}]$ for which \robust outperforms the baseline schedule. 
}\label{tableRobust}        
\end{table}

%% file: tables.tex
\newcommand{\tableInterruptionAsAdvice}{
}

\newcommand{\tableQuerySettingLargeP}{

\begin{table}[!h]
        \centering
        
\scalebox{.8}{
\begin{tabular}{|p{0.20\columnwidth}|p{0.13\columnwidth}|p{0.17\columnwidth}|p{0.17\columnwidth}|p{0.17\columnwidth}|p{0.17\columnwidth}|}
\hline 
  & $\displaystyle \eta \ =\ 0.1$ & $\displaystyle \eta \ =\ 0.15$ & $\displaystyle \eta \ =\ 0.2$ & $\displaystyle H\ =\ 0.2$5 & $\displaystyle H\ =\ 0.30$ \\
\hline 
 advantage & 0.983 & 0.983 & 0.983 & 0.976 & 0.963 \\
\hline 
 strong advantage & 0.703 & 0.673 & 0.610 & 0.510 & 0.366 \\
 \hline
\end{tabular}
}\caption{In the query setting, the adaptive algorithm outperform baseline algorithm for almost all values of deadline. The top (respectively bottom) row indicate the percentage of the interruption times for which the adaptive algorithm completes a longer contract (respectively a contract that is at least 20 percent longer) than the baseline algorithm.}        
        \end{table}
}

\newcommand{\tableQuerySettingSmallP}{

\begin{table}[!h]
        \centering
        
\scalebox{.8}{
\begin{tabular}{|p{0.25\columnwidth}|p{0.25\columnwidth}|p{0.25\columnwidth}|p{0.25\columnwidth}|}
\hline 
  & $\displaystyle p\ =\ 0.0$ & $\displaystyle p\ =\ 0.05$ & $\displaystyle p\ =\ 0.07$ \\
\hline 
 advantage & 0.623 & 0.813 & 0.876 \\
\hline 
 strong advantage & 0.520 & 0.696 & 0.723 \\
 \hline
\end{tabular}
}\caption{In the query setting, the oblivious algorithms (with $p<\eta$) still outperform baseline algorithm for most values of deadline. The top (respectively bottom) row indicate the percentage of the interruption times for which the oblivious algorithms completes a longer contract (respectively a contract that is at least 20 percent longer) than the baseline algorithm.}        /        
        \end{table}

}

\newcommand{\tableQueryAdaptive}{

\begin{table}[!h]
        \centering        
\scalebox{.77}{
\begin{tabular}{|p{0.20\columnwidth}|p{0.13\columnwidth}|p{0.17\columnwidth}|p{0.17\columnwidth}|p{0.17\columnwidth}|p{0.17\columnwidth}|}
\hline 
  & H = 0.0 & H = 0.05 & H = 0.1 & H = 0.2 & H = 0.3 \\
\hline 
 advantage & 0.983 & 0.953 & 0.936 & 0.883 & 0.810 \\
\hline 
 strong advantage & 0.776 & 0.736 & 0.693 & 0.593 & 0.453 \\
 \hline
\end{tabular}}
        \caption{The percentage of deadlines for which the adaptive \textsc{Robust} schedule completes a longer contract than the baseline algorithm (top) or completes a contract that is at least 20 percent longer than the baseline algorithm (bottom).\label{tableQueryAdaptive} }
        \end{table}

}

\newcommand{\tableQueryObli}{

\begin{table}[!h]
        \centering        
\scalebox{.77}{
\begin{tabular}{|p{0.20\columnwidth}|p{0.13\columnwidth}|p{0.17\columnwidth}|p{0.17\columnwidth}|p{0.17\columnwidth}|p{0.17\columnwidth}|}
\hline 
  & p = 0.0 & p = 0.05 & p = 0.1 & p = 0.2 & p = 0.3 \\
\hline 
 advantage & 0.703 & 0.896 & 0.936 & 0.853 & 0.750 \\
\hline 
 strong advantage & 0.593 & 0.723 & 0.696 & 0.586 & 0.476 \\
 \hline
\end{tabular}}
        \caption{The percentage of deadlines for which the oblivious \robust schedules with different parameters $p$ complete a longer contract than the baseline algorithm (top) or complete a contract that is at least 20 percent longer than the baseline algorithm (bottom).\label{tableQueryObli} }
        \end{table}

}